\documentclass{article}

\usepackage{fullpage}
\usepackage[small]{caption}
\usepackage[utf8]{inputenc} 
\usepackage[T1]{fontenc}    
\usepackage[hidelinks]{hyperref}       
\usepackage{url}            
\usepackage{booktabs}       
\usepackage{amsfonts}       
\usepackage{nicefrac}       
\usepackage{microtype}      
\usepackage{algorithm}
\usepackage{algpseudocode}
\usepackage{amsmath}
\usepackage{amssymb}
\usepackage{amsthm}
\usepackage{algorithm}
\usepackage{algpseudocode}
\usepackage{mathtools}
\usepackage{breakcites}
\usepackage{csquotes}
\usepackage{colortbl}
\usepackage{natbib}
\usepackage[shortlabels,inline]{enumitem}
\usepackage{arydshln}
\usepackage{mathabx}

\makeatletter
\newcommand{\inlineitem}[1][]{%
\ifnum\enit@type=\tw@
    {\descriptionlabel{#1}}
  \hspace{\labelsep}%
\else
  \ifnum\enit@type=\z@
       \refstepcounter{\@listctr}\fi
    \quad\@itemlabel\hspace{\labelsep}%
\fi}
\makeatother
\usepackage{thm-restate}
\usepackage{tabularx}
\usepackage{xcolor}
\usepackage{xspace}

 \usepackage{natbib}
 \usepackage{bm}

\newtheorem{lemma}{Lemma}
\newtheorem{theorem}[lemma]{Theorem}

\newtheorem{informal theorem}[lemma]{Informal Theorem}
\usepackage[shortlabels,inline]{enumitem}

\newcommand{\poly}{\mbox{poly}}
\newcommand{\diag}{\mbox{diag}}
\newcommand{\nnz}{\mbox{nnz}}

\newcommand{\eps}{\ensuremath{\varepsilon}}

\newcommand{\remove}[1]{}

\usepackage{colortbl}%

\title{Sublinear Time Numerical Linear Algebra for Structured Matrices}

\author
{Xiaofei Shi, David P. Woodruff\\Carnegie Mellon University \\ {\texttt xiaofei.shi@andrew.cmu.edu, dwoodruf@cs.cmu.edu}}
\date{}

\begin{document}

\maketitle

\begin{abstract}
We show how to solve a number of problems in numerical linear
algebra, such as least squares regression, $\ell_p$-regression for any $p \geq 1$, 
low rank approximation, and kernel regression, in time $T(A)   \poly(\log(nd))$, 
where for a given input matrix $A \in \mathbb{R}^{n \times d}$, 
$T(A)$ is the time needed to compute 
$A\cdot y$ for an arbitrary vector $y \in \mathbb{R}^d$. Since $T(A) \leq O(\nnz(A))$,
where $\nnz(A)$ denotes the number of non-zero entries of $A$, the time is no worse,
up to polylogarithmic factors, as all of the recent advances for such problems that
run in input-sparsity time. However, for many applications, $T(A)$ can be much
smaller than $\nnz(A)$, yielding significantly sublinear time algorithms. 
For example, in the overconstrained $(1+\epsilon)$-approximate polynomial 
interpolation problem, $A$ is a Vandermonde
matrix and $T(A) = O(n \log n)$; in this case our running time is
$n   \cdot \poly(\log n) + \poly(d/\epsilon)$
and we recover the results of \cite{avron2013sketching} as a special case. For overconstrained  
autoregression, which is a common problem arising in dynamical systems, 
$T(A) = O(n \log n)$, and 
we immediately obtain $n \cdot \poly(\log n) + \poly(d/\epsilon)$ time. 
For kernel autoregression, we significantly improve the running time
of prior algorithms for general kernels. 
For the important case of autoregression with the polynomial kernel and arbitrary target vector $b\in\mathbb{R}^n$, 
we obtain even faster algorithms. Our algorithms 
show that, perhaps surprisingly, most of these optimization problems do not require much more
time than that of a polylogarithmic number of matrix-vector multiplications. 
\footnote{A first version of this paper appeared in AAAI in February, 2019.}
\end{abstract}

\section{Introduction}
\label{sec:intro}
A number of recent advances in randomized numerical linear algebra have been made possible by the technique of oblivious sketching. In this setting, given an $n \times d$ input  matrix $A$ to some problem, one first computes a sketch $S   A$ where $S$ is a random matrix drawn from a certain random family of matrices. Typically $S$ is wide and fat, and therefore applying $S$ significantly reduces the number of rows of $A$. Moreover, $S   A$ preserves structural information about $A$. 

For example, in the least squares regression problem one is given an $n \times d$ matrix $A$ and an $n \times 1$ vector $b$ and one would like to output a vector $x \in \mathbb{R}^d$ for which
\begin{align}\label{eqn:regression}
\|Ax-b\|_2 \leq (1+\epsilon)\min_x \|Ax-b\|_2,
\end{align}
where for a vector $y$, $\|y\|_2 = \left (\sum_i |y_i|^2 \right )^{1/2}$. Typically the $n$ rows of $A$ correspond to observations, and one would like the prediction $\langle A_i, x \rangle$ to be close to the observation $b_i$, where $A_i$ denotes the $i$-th row of $A$. While it can be solved exactly via the normal equations, one can solve it much faster using oblivious sketching. Indeed, for overconstrained least squares where $n \gg d$, one can choose $S$ to be a subspace embedding, meaning that simultaneously for all vectors $x \in \mathbb{R}^d$, $\|SAx\|_2 = (1 \pm \epsilon) \|Ax\|_2$. In such applications, $S$ has only $\poly(d/\epsilon)$ rows, independent of the large dimension $n$. By computing $S A$ and $S b$, and solving $x' = \textrm{argmin}_x \|SAx-Sb\|_2$, one has that $x'$ satisfies (\ref{eqn:regression}) with high probability. Thus, much of the expensive computation is reduced to the ``sketch space'', which is independent of $n$. 

Another example is low rank approximation, in which one is given an $n \times d$ matrix $A$ and would like to find an $n \times k$ matrix $U$ and a $k \times d$ matrix $V$ so that
\begin{align}\label{eqn:lowRank}
\|U   V - A\|_F^2 \leq (1+\epsilon)\|A-A_k\|_F^2,
\end{align}
where for a matrix $B \in \mathbb{R}^{n \times d}$, $\|B\|_F^2 = \sum_{i =1}^n \sum_{j=1}^d B_{i,j}^2$, and where $A_k = \textrm{argmin}_{\textrm{rank-} k \ B} \|A-B\|_F^2$ is the best rank-$k$ approximation to $A$. While it can be solved via the singular value decomposition (SVD), one can solve it much faster using oblivious sketching. In this case one chooses $S$ so that the row span of $S   A$ contains a good rank-$k$ space, meaning that there is a matrix $V \in \mathbb{R}^{k \times d}$ whose row span is inside of the row span of $S   A$, so that there is a $U$ for which this pair $(U,V)$ satisfies the guarantee of
(\ref{eqn:lowRank}). Here $SA$ only has $\poly(k/\epsilon)$ rows, independent of $n$ and $d$. Several known algorithms approximately project the rows of $A$ onto the row span of $SA$, then compute the SVD of the projected points to find $V$, and then solve a regression problem to find $U$. Other algorithms compute the top $k$ directions of $SA$ directly. Importantly, the expensive computation involving the SVD can be carried out in the much lower $\poly(k/\epsilon)$-dimensional space rather than the original $d$-dimensional space. 

While there are numerous other examples, such as $\ell_p$-regression and kernel variations of the above problems (see \cite{w14} for a survey), they share the same flavor of first reducing the problem to a smaller problem in order to save computation. For this reduction to be effective, the matrix-matrix product $S   A$ needs to be efficiently computable. One typical sketching matrix $S$ that works is a matrix of i.i.d. Gaussians; however since $S$ is dense, computing $S   A$ is slow. Another matrix which works is a so-called fast Johnson-Lindenstrauss transform, see~\cite{s06}. As in the Gaussian case, $S$ has a very small number of rows in the above applications, and computing $SA$ can be done in $\tilde{O}(nd)$ time, where $\tilde{O}(f)$ denotes a function of the form $f \cdot \poly(\log f)$. This is useful if $A$ is dense, but often $A$ is sparse and may have a number $\nnz(A)$ of non-zero entries which is significantly smaller than $nd$. Here one could hope to compute $S   A$ in $\nnz(A)$ time, which is indeed possible using a CountSketch matrix, see~\cite{cw13,mm13,nn13}, also with a small number of rows. 

For most problems in numerical linear algebra, one needs to at least read all the non-zero entries of $A$, as otherwise one could miss reading a potentially very large entry. For example, in the low rank approximation problem, if there is one entry which is infinite and all other entries are small, the best rank-$1$ approximation would first fit the single infinite-valued entry. From this perspective, the above $\nnz(A)$-time algorithms are optimal. However, there are many applications for which $A$ has additional structure. For example, the polynomial interpolation problem is a special case of regression in which the matrix $A$ is a Vandermonde matrix. As observed in \cite{avron2013sketching}, in this case if $S \in \mathbb{R}^{\textrm{poly}(d/\epsilon) \times n}$ is a CountSketch matrix, then one can compute $S   A$ in $O(n \log n) + \poly(d/\epsilon)$ time. This is {\it sublinear} in the number of non-zero entries of $A$, which may be as large as $nd$ thus may be much larger than $O(n \log n) + \poly(d/\epsilon)$. The idea of \cite{avron2013sketching} was to simultaneously exploit the sparsity of $S$, together with the fast multiplication algorithm based on the Fast Fourier Transform associated with Vandermonde matrices to reduce the computation of $S   A$ to a small number of disjoint matrix-vector products. A key fact used in \cite{avron2013sketching} was that submatrices of Vandermonde matrices are also Vandermonde, which is a property that does not hold for other structured families of matrices, such as Toeplitz matrices, which arise in applications like autoregression. There are also sublinear time low rank approximation algorithms of matrices with other kinds of structure, like PSD and distance matrices, see~\cite{mw17,bw18}. 

An open question, which is the starting point of our work, is if one can extend the results of \cite{avron2013sketching} to {\it any structured matrix} $A$. More specifically, {\it 
can one solve all of the aforementioned linear algebra problems in time $T(A)$ instead of $\nnz(A)$
, where $T(A)$ is the time required to compute $A   y$ for a single vector $y$? }
For many applications, discussed more below, one has a structured matrix $A$ with $T(A) =O( n \log n) \ll \nnz(A)$. 

{\bf Our Contributions.}
We answer the above question in the affirmative, showing that for a number of problems in numerical linear algebra, one can replace the $\nnz(A)$ term with a $T(A)$ term in the time complexity. Perhaps surprisingly, we are not able to achieve these running times via oblivious sketching, but rather need to resort to sampling techniques, as explained below. 
We state our formal results:
\begin{itemize}
\item{\it Low Rank Approximation:} Given an $n \times d$ matrix $A$, we can find $U \in \mathbb{R}^{n \times k}$ and $V \in \mathbb{R}^{k \times d}$ satisfying (\ref{eqn:lowRank}) in $O\left(T(A)  \log n + n \cdot \poly(k/\epsilon)\right)$ time. 

\item {\it $\ell_p$-Regression:} Given an $n \times d$ matrix $A$ and an $n \times 1$ vector $b$, one would like to output an $x \in \mathbb{R}^d$ 
for which 
\begin{align}\label{eqn:lp}
\|Ax-b\|_p \leq (1+\epsilon) \min_x \|Ax-b\|_p,
\end{align}
where for a vector $y$, $\|y\|_p = (\sum_i |y_i|^p)^{1/p}$. 
We show for any real 
number $p \geq 1$, we can solve this problem in $O\left(T(A)\log n + \poly(d/\epsilon)\right)$ time. This includes least squares regression ($p = 2$) as a special case.

\item {\it Kernel Autoregression:} 
A kernel function is a mapping $\phi:\mathbb{R}^p \rightarrow \mathbb{R}^{p'}$ where $p' \geq p$ so that the inner product $\langle \phi(x), \phi(y) \rangle$ between any two points $\phi(x), \phi(y) \in \mathbb{R}^{p'}$ can be computed quickly given the inner product $\langle x, y \rangle$ between $x, y \in \mathbb{R}^{p}$. Such mappings are useful when it is not possible to find a linear relationship between the input points, but after lifting the points to a higher dimensional space via $\phi$ it become possible. We are given a matrix $A\in\mathbb{R}^{n p \times d}$ for which the rows can be partitioned into $n$ contiguous $p \times d$ block matrices $A^1, \ldots, A^n$. Further, we are in the setting of autoregression, so  
for $j = 2, \ldots, n$, $A^j$ is obtained from $A^{j-1}$ by setting the $\ell$-th column $A^j_{\ell}$ of $A^j$ to be the $(\ell-1)$-st column of $A^{j-1}$, namely, to $A^{j-1}_{\ell-1}$. The first column $A^j_1$ of $A^j$ is allowed to be arbitrary. 
Let $\phi(A)$ be the matrix obtained from $A$ by replacing each block $A^j$ with $\phi(A^j)$, where $\phi(A^j)$ is obtained from $A^j$ by replacing each column $A^j_{\ell}$ with $\phi(A^j_{\ell})$. 
We are also given an $(n   p') \times 1$ vector $b$, perhaps implicitly. 
We want $x \in \mathbb{R}^d$ to minimize $\| \phi(A)x-b\|_2$. 
\end{itemize}

For general kernels not much is known, though prior work \cite{kumar2007kernel} shows how to find a minimizer $x$ assuming i.i.d. Gaussian noise. Their running time is $O(n^2 t)$, where $t$ is the time to evaluate $\langle \phi(x), \phi(y) \rangle$ given $x$ and $y$. We show how to improve this to $O(ndt + d^{\omega})$ time, 
where $\omega \approx 2.376$ is the exponent of fast matrix multiplication. Note for autoregression that $b$ has the form $[\phi(c^1); \phi(c^2); \ldots; \phi(c^n)]$ for certain vectors $c^1, \ldots, c^n$ that we know. As $n \gg d$ in overconstrained regression, our $O(ndt + d^{\omega})$ time is faster than the $O(n^2 t + d^{\omega})$ time of earlier work. For dense matrices $A$, describing $A$ already requires $\Omega(ndp)$ time, so in the typical case 
when $t \approx p$, we are optimal for such matrices.
We note that prior work \cite{kumar2007kernel} assumes Gaussian noise, while we 
do not make such an assumption. 

While the above gives an improvement for general kernels, one could also hope for much faster algorithms. In general we would like an $x$ for which:
\begin{align}\label{eqn:regress}
\|\phi(A)x - b\|_2 \leq (1+\epsilon)\min_x \|\phi(A)x-b\|_2.
\end{align}
We show how to solve this in the case that $\phi$ corresponds to the polynomial kernel of degree $2$, though discuss extensions to $q > 2$. 
In this case, $\langle \phi(x), \phi(y) \rangle = \langle x, y \rangle^q$. 
The running time of our algorithm is $O(\nnz(A)) + \textrm{poly}(pd/\epsilon)$. Note that $b$ is an {\it arbitrary} $n   p'$-dimensional vector, and our algorithm runs in sublinear time in the length of $b$ - this is possible by judiciously sampling certain coordinates of $b$. Note even for dense matrices, $\nnz(A) \leq n p d$, which does not depend on the large value $p'$. We also optimize the $\poly(pd/\epsilon)$ term. 

{\bf Applications.} Our results are quite general, recovering the results of \cite{avron2013sketching} for Vandermonde matrices which have applications to polynomial fitting and additive models as a special case. We refer the reader to \cite{avron2013sketching} for details, and here we focus on other implications. One application is to autoregression, which is a time series model which uses observations from previous time 
steps as input to a regression problem to predict the value in the next time step, and
can provide accurate forecasts on time series problems. It is
often used to model stochastic time-varying processes in nature, economics, etc. 
Formally, in autoregression we have:
\begin{align}{\label{eqn:ar}}
b_t = \sum_{i=1}^d b_{t-i} x_i + \epsilon_t,
\end{align}
where $d +1\leq t \leq n+d$, 
and the $\epsilon_t$ correspond to the noise in the model. 
We note that $b_1$ is defined to be $0$. 
This model is known as the $d$-th order Autoregression model (AR(d)). 

The underlying matrix in the AR(d) model corresponds to the first $d$
columns of a {\it Toeplitz matrix}, and consequently one can compute $A^TA$
in $O(n d \log n)$ time, which is faster than the $O(nd^{\omega-1})$ time which assuming $d > \poly(\log n)$
for computing $A^T A$ for general matrices $A$, where $\omega \approx 2.376$ is the exponent of matrix multiplication. 
Alternatively, one can apply
the above sketching techniques which run in time 
$O(\nnz(A)) + \poly(d/\epsilon) = O(nd) + \poly(d/\epsilon)$. Either way, this gives a time of $\Omega(nd)$. We show
how to solve this problem in $O(n \log^2 n) + \poly(d/\epsilon)$
time, which is a significant improvement over the above methods whenever 
$d>\log^2 n $. There are a number of other works on Toeplitz linear systems and regression see~\cite{rom8,rom10,rom11,rom12,rom13,rom14}; our work is the first row sampling-based algorithm, and this technique will be crucial for obtaining our polynomial kernel results. More generally,
our algorithms only depend on $T(A)$, rather than on specific properties of $A$. If instead of just a Toeplitz
matrix $A$, one had a matrix of the form $A+B$, where $B$ is an arbitrary matrix with $T(B) = O(n \log n)$, 
e.g., a sparse perturbation to $A$, we would obtain the same running time. 

Another stochastic process model is
the vector autoregression (VAR), in which one replaces 
the scalars $b_t\in\mathbb{R}$ in (\ref{eqn:ar}) with points in $\mathbb{R}^p$. This
forecast model is used in Granger causality, impulse responses, forecast error
variance decompositions, and health research \cite{health}. An extension is
kernel autoregression \cite{kumar2007kernel}, where we additionally have a kernel function 
$\phi: \mathbb{R}^p \rightarrow \mathbb{R}^{p'}$ with $p' > p$,  
and further replace $b_t$ with $\phi(b_t)$ in (\ref{eqn:ar}). 
One wants to find
the coefficients $x_1, \ldots, x_d$ fitting the points $\phi(b_t)$ without
computing $\phi(b_t)$, which may not be possible since $p'$ could be very large
or even infinite. To the best of our knowledge, our results 
give the fastest known algorithms for VAR and kernel autoregression. 

{\bf Our Techniques.}
Unlike the result in \cite{avron2013sketching} for Vandermonde matrices, many of our results for other structured matrices {\it do not use oblivious sketching}. We illustrate the difficulties for least squares regression of using oblivious sketching. In \cite{avron2013sketching}, given an $n \times d$ Vandermonde matrix $A$, one wants to compute $S   A$, where $S$ is a CountSketch matrix. 
For each $i$, the $i$-th row of $A$ has the form $(1, x_i, x_i^2, \ldots, x_i^{d-1})$. $S$ has $r = \poly(d/\epsilon)$ rows and $n$ columns, and each column of $S$ has a single non-zero entry located at a uniformly random chosen position. Denote the entry in the $i$-th column as $h(i)$, then $SA$ decomposes into $r$ matrix-vector products, where each row of $A$ participates in exactly one matrix product. Namely, we can group the rows of $A$ into submatrices $A^i$ and create a vector $x^i$ which indexes the subset of coordinates $j$ of $x$ for which $h(j) = i$. The $i$-th row of $S   A$ is precisely $x^i   A^i$. For a submatrix $A^i$ of a Vandermonde matrix, the product $x^i   A^i$ can be computed in $O(s_i \log s_i)$ time, where $s_i$ is the number of rows of $A^i$. The total time to compute $S   A$ is thus $O(n \log n)$. 

Now suppose $A \in \mathbb{R}^{n \times d}$, $d \ll n$, is a rectangular Toeplitz matrix, i.e., the $i$-th row of $A$ is obtained by shifting the $(i-1)$-st row to the right by one position, and including an arbitrary entry in the first position. Toeplitz matrices are the matrices which arise in autoregression. We can think of $A$ as a submatrix of a square Toeplitz matrix $C$, and can compute $x   C$ for any vector $x$ in $O(n \log n)$ time. Unfortunately though, an $r \times d$ submatrix $A^i$ of a Toeplitz matrix, $r > d$, may not have an efficient multiplication algorithm. Indeed, imagine the $r$ rows correspond to disjoint subsets of $d$ coordinates of a Toeplitz matrix. Then computing $x   A^i$ would take $O(rd)$ time, whereas for a Vandermonde matrix one could always multiply a vector times an $r \times d$ submatrix in only $O(r \log r)$ time. Vandermonde matrices are a special sub-class of structured matrices which are closed under taking sub-matrices, which we do not have in general.

Rather than using oblivious sketching, we instead use {\it sampling-based techniques}. A
first important observation is that the sampling-based techniques for subspace approximation \cite{c14},
low rank approximation \cite{cohen2017input}, 
and $\ell_p$-regression \cite{cohen2015p}, can each be implemented
with only $t = O(\log n)$ matrix-vector products between the input matrix $A$ and certain arbitrary vectors
$v^1, \ldots, v^t$ arising throughout the course of the algorithm. We start by verifying this property
for each of these important applications, allowing us to replace the $\nnz(A)$ term with a $T(A)$ term. We
then give new algorithms for autoregression, for which the design matrix is a truncated
Toeplitz matrix, and more generally composed with a difference and a diagonal matrix. 

Our technically
more involved results are then for kernel autoregression. 
First for general kernels, we show how to accelerate the computation of $\phi(A)^T \phi(A)$ using the Toeplitz
nature of autoregression, 
and observe that only $O(nd)$ inner products ever need to be
computed, even though there are $\Theta(n^2)$ possible inner products. 
We then show how to solve polynomial kernels of degree $q$. We focus on 
$q = 2$ though our arguments can be extended to $q > 2$.  
We first use oblivious sketching to compute
a $d \times O(\log n)$ matrix $RG$ from which, via standard arguments, it suffices to sample $O(d \log d + d/\epsilon)$ 
row indices $i$ proportional to $\|e_i \phi(A) RG\|_2^2$, where $e_i$ is the $i$-th standard
unit vector. 
Given the sampled row indices $i$,
one can immediately find the $i$-th row of $\phi(A)$, 
since the index $i$ corresponds to a $q$-tuple $(i_1, \ldots, i_q)$ of a block $\phi(A^j)$ with columns
$\phi(A^j_{\ell})$, for $\ell \in \{1, 2, \ldots, d\}$, and so 
$e_i   \phi(A) e_k = A^j_{i_1, k}   A^j_{i_2, k} \cdots A^j_{i_q, k}$. We can also directly read off 
the corresponding entry from $b$. The $j$-th row of $S$ is also just $\sqrt{\frac{1}{p_i}}   e_i$ if
row index $i$ is the $j$-th sampled row, where $p_i$ is the probability of sampling $i$. Further, 
the matrices $R   G$ and $S$ can be found in $O(\nnz(A) + d^3)$ time using earlier work \cite{cw13,anw14}. We show to find the set of $O(d \log d + d/\epsilon)$ sampled row indices quickly.
Here we use that 
 $\phi(A)$ is ``block Toeplitz'', together
with a technique of replacing blocks of $\phi(A)$ with ``sketched blocks'', which allows us to sample blocks of $\phi(A) RG$ proportional to their squared norm. We then need to obtain a sampled index inside of a block, 
and for the polynomial kernel of degree $2$ 
we use the fact that the entries of $\phi(A^j)   y$ for a vector $y$
are in one-to-one correspondence with the entries of $A^{j-1}   D_y   (A^{j-1})^T$, 
where $D_y$ is a diagonal matrix with $y$ along the diagonal. We do not need to compute
$A^{j-1}   D_y   (A^{j-1})^T$, but can compute $H   A^{j-1}   D_y   (A^{j-1})^T$
for a matrix $H$ of i.i.d. Gaussians in order to sample a column of $A^{j-1}   D_y   (A^{j-1})^T$
proportional to its squared norm, after which we can compute the sampled column exactly and output an entry
of the column proportional to its squared value. 
Here we use the Johnson Lindenstrauss lemma to argue
that $H$ preserves column norms. 
A similar identity holds for degrees $q > 2$, and that identity 
was used in the context of Kronecker product regression
\cite{dswy19}. 

\section{Fast Algorithms Based on Sampling}\label{fast algo}
We first consider $\min_x \|Ax-b\|_2$, where $A \in \mathbb{R}^{n \times d}, b \in \mathbb{R}^{n \times 1}$, and $n > d$. We show how, in
$O(T(A)   \poly(\log n) + \poly(d (\log n)/\epsilon))$ time, to reduce this to a 
problem $\min_x \|SAx-Sb\|_2$, where 
$SA \in \mathbb{R}^{r \times d}$ and $Sb \in \mathbb{R}^{r \times 1}$
such that if $\hat{x} = \textrm{argmin}_x \|SAx-Sb\|_2$, then
\begin{align}\label{eqn:guarantee}
\|A\hat{x}-b\|_2 \leq (1+\epsilon) \min_x \|Ax-b\|_2.
\end{align} 
Here
$r = O(d/\epsilon^2)$. Given $SA$ and $Sb$, one can compute $\hat{x} = (SA)^- Sb$ in
$\poly(d/\epsilon)$ time. For (\ref{eqn:guarantee}) to hold, 
it suffices for the matrix
$S$ to satisfy the property that for all $x$, 
$\|SAx-Sb\|_2 = (1\pm \epsilon)\|Ax-b\|_2$. This is 
implied if for any fixed $n \times (d+1)$ matrix $C$, 
$\|SCx\|_2 = (1 \pm \epsilon)\|Cx\|_2$ for all $x$.
Indeed, in this case we may set $C = [A, b]$. 
This problem is sometimes called the {\it matrix approximation} problem.

\noindent{\bf The {\sc Repeated Halving} Algorithm.}
The following
algorithm for matrix approximation is given in \cite{c14}, and called
{\sc Repeated Halving}.

\begin{algorithm}[H]
\caption{Repeated Halving}
\begin{algorithmic}[1]
\Procedure{RepeatedHalving}{$C \in \mathbb{R}^{n \times (d+1)}$}
\State \hspace{-4mm}{Uniformly sample $n/2$ rows of $C$ to form $C'$}
\State \hspace{-4mm}{If $C'$ has more than $O((d \log d) / \epsilon^2)$ rows, recursively
compute a spectral approximation $\tilde{C}'$ of $C'$ }
\State \hspace{-4mm}{Approximate generalized leverage scores of $C$ w.r.t. $\tilde{C}'$}
\State \hspace{-4mm}{Use these estimates to sample rows of $C$ to form $\tilde{C}$}
\State \hspace{-4mm}\Return $\tilde{C}$
\EndProcedure
\end{algorithmic}
\end{algorithm}
\noindent{\bf Leverage Score Computation.}
We first clarify step 4 in {\sc Repeated Halving}, which is a 
standard Johnson-Lindenstrauss trick for speeding up leverage score
computation \cite{DMMW12}. The $i$-th
generalized leverage score of a matrix $C$ with $n$ rows 
w.r.t. a matrix $B$ is defined
to be $\tau_i^{B}(C) = c_i^T (B^T B)^+c_i = \|B(B^TB)^+c_i\|_2^2$, where $c_i$
is the $i$-th row of $C$, written as a column vector. The idea
is to instead compute $\|GB(B^TB)^+c_i\|_2^2$, where $G$ is a random
Gaussian matrix with $O(\log n)$ rows. The Johnson-Lindenstrauss
lemma and a union bound yield 
$\|GB(B^TB)^+c_i\|_2^2 = \Theta(1)\|B(B^TB)^+c_i\|_2^2$. If $B$ is 
$O((d \log d) / \epsilon^2) \times d$, then $(B^TB)^+$ can be computed
in $\poly(d/\epsilon)$ time. We compute
$GB$ in $O((d^2/\epsilon^2)\log n)$ time, 
which is an $O(\log n) \times d$ matrix.
Then we compute $(GB)   (B^TB)^+$, which now takes only $O(d^2 \log n)$
time, and is an $O(\log n) \times d$ matrix. Finally one can compute
$GB (B^TB)^+ C^T$ in $O(\nnz(C)\log n)$ time, and the squared column
norms are constant factor approximations to the $\tau_i^{B}(C)$ values. 
The total time
to compute all $i$-th generalized leverage scores is
$O(\nnz(C) \log n) + \poly(d \log n / \epsilon)$. 

\noindent{\bf Sampling. }
We clarify how step 5 in {\sc Repeated Halving} works, 
which is a standard
leverage score sampling-based procedure, see, e.g., \cite{m11}. 
Given a list of approximate
generalized leverage scores $\tilde{\tau}_i^{B}(C)$, we sample
$O((d \log d)/\epsilon^2)$ rows of $C$ independently proportional to 
form $\tilde{C}$.
We write this as $\tilde{C} = S   C$, where the $i$-th row
of $S$ has a $1/\sqrt{p_{j(i)}}$ in the $j(i)$-th position, where
$j(i)$ is the row of $C$ sampled in the $i$-th trial, and 
$p_{j(i)} = \tilde{\tau}_i^{B}(C)/\sum_{i'=1, \ldots, n} \tilde{\tau}_{i'}^B(C)$
is the probability of sampling $j(i)$ in the $i$-th trial. Here
$S$ is called a {\it sampling and rescaling} matrix.
Sampling independently from a distribution
on $n$ numbers with replacement $O((d \log d)/\epsilon^2)$ times can be done
in $O(n + (d \log d)/\epsilon^2)$ time \cite{v91}, giving a total
time spent in step 5 
of $O(n \log n + (d \log d)(\log n)/\epsilon^2)$ across all $O(\log n)$
recursive calls. 
As argued in \cite{c14}, the error probability is at most $1/100$, which
can be made an arbitrarily small constant by appropriately setting the
constants in the big-Oh notation above. 

\noindent{\bf Speeding up {\sc Repeated Halving}.}
We now show how to speed up the {\sc Repeated Halving} algorithm.
Step 2 of {\sc Repeated Halving} can be implemented just by choosing
a subset of {\it row indices} in $O(n)$ time. Step 3 just involves
checking if the number of uniformly sampled rows 
is larger than $O((d \log d)/\epsilon^2)$, which can be done in constant
time, and if so, a recursive call is performed. The number of recursive
calls is at most $O(\log n)$, since Step 1 halves the number of rows. So
the total time spent on these steps is 
$O(n \log n + (d \log d) (\log n)/\epsilon^2)$. 

In step 4 of {\sc Repeated Halving}, we compute generalized leverage
scores of $C$ with respect to a matrix $\tilde{C}'$, and is only 
(non-recursively) applied when $\tilde{C'}$ has $O((d \log d)/\epsilon^2)$
rows. As described when computing leverage scores with $B = \tilde{C'}$, 
we must do the following:
\begin{enumerate}
\item  Compute $G\tilde{C'} \in\mathbb{R}^{O(\log n) \times d}$ in time $O((d^2/\epsilon^2)\log n)$

\item Compute $((\tilde{C'})^T \tilde{C'})^+$ in time $O(d^3/\epsilon^2)$ 

\item Compute $(G \tilde{C'})   ((\tilde{C'})^T \tilde{C'})^+$ in time $O(d^2 \log n)$

\item Compute $G \tilde{C'} ((\tilde{C'})^T \tilde{C'})^+ C^T$ 
\end{enumerate}
Since $G\tilde{C'} ((\tilde{C'})^T \tilde{C'})^+$ has 
$O(\log n)$ rows that already be computed, one can compute $G \tilde{C'} ((\tilde{C'})^T \tilde{C'})^+ C^T$ 
in $O(\log n)   T(C)$ time,
where $T(C)$ is the time needed to multiply $C$ by a vector (note that computing
$yC^T$ is equivalent to computing $Cy^T$).   
In our application to regression, $C = [A, b]$. Consequently,
$T(C) \leq T(A) + n$. As the number of recursive calls is $O(\log n)$, it follows that the total time spent for step 4 of {\sc Repeated Halving}, 
across all recursive calls, is $O(T(A) \log n + n \log n + (d^2 \log^2 n)/\epsilon^2 + d^3 (\log n)/\epsilon^2)$. 
The fifth step of {\sc Repeated Halving} is to find the sampling
and rescaling matrix as described above, which can be done in 
$O(n \log n + (d \log d)(\log n)/\epsilon^2)$ total time, across
all recursive calls. Thus, the total time is
$O(T(A) \log n + n \log n + (d^2 \log^2 n)/\epsilon^2 + d^3 (\log n)/\epsilon^2).$
We summarize our findings with the following theorem.
\begin{theorem}\label{thm:main}
Given an $n \times d$ matrix $A$, an $n \times 1$ vector $b$, an
accuracy parameter $0 < \epsilon < 1$, and a failure probability
bound $0 < \delta < 1$,  
one can output a vector $\hat{x} \in \mathbb{R}^d$ for which
$\|A\hat{x}-b\|_2 \leq (1+\epsilon) \min_x \|Ax-b\|_2$ with probability
at least $1-\delta$, in total time
\begin{align*}
O((T(A) \log n +  \textrm{poly}(d\log n/\epsilon))\log(1/\delta)).
\end{align*}
\end{theorem}
\begin{proof}
From the discussion above, our modified version of {\sc Repeated Halving} produces a vector 
$\hat{x}$ for $\|A\hat{x}-b\|_2 \leq (1+\epsilon) \min_x \|Ax-b\|_2$
with probability at least $99/100$. Repeating $r = O(\log(1/\delta))$
times independently, obtaining candidate solutions $\hat{x}^1, \ldots, \hat{x}^r$, and choosing the $\hat{x}^i$ for which $\|A\hat{x}^i-b\|_2$ is smallest,
one reduces the failure probability to $\delta$ via standard Chernoff bounds. The time to compute $\|A\hat{x}^i-b\|_2$ given $\hat{x}^i$ is at
most $T(A) + O(n) $, which is negligible compared to other
operations in a repetition. 
\end{proof}

{\bf Low Rank Approximation.}
We look at the {\it low-rank approximation} problem, where for $A \in\mathbb{R}^{n\times d}$ one tries to find a
matrix $Z\in \mathbb{R}^{n\times k}$ with orthonormal columns such that
\begin{align}\label{eq:low rank}
\| A - Z Z^TA \|_F^2 \leq (1+\eps)\|A - A_k\|_F^2.
\end{align}

Here, $A_k$ is the best rank $k$ approximation to $A$. 
It is shown in \cite{cohen2015dimensionality} that the low-rank approximation problem can be solved by finding a 
subset of rescaled columns $C\in\mathbb{R}^{n\times d'}$ with $d'<d$, such that for every rank $k$ orthogonal projection matrix $X$:
\begin{align}\label{eq:low rank via leverage}
\|C - XC\|_F^2 = (1\pm\eps)\|A - XA\|_F^2. 
\end{align}

{\bf Basic Recursive Algorithm.}
In \cite{cohen2017input}, a slightly different version of Algorithm 1 with ridge leverage score approximation is used to solve~\eqref{eq:low rank via leverage}:

\begin{algorithm}[H]
\caption{Repeated Halving}
\begin{algorithmic}[1]
\Procedure{RepeatedHalving}{$A \in \mathbb{R}^{n \times d}$}
\State \hspace{-4mm}{Uniformly sample $d/2$ columns of $A$ to form $C'$}
\State \hspace{-4mm}{If $C'$ has more than $O(k \log k)$ columns, recursively
compute a constant approximation $\tilde{C}'$ for $C'$ with $O(k \log k)$ columns }
\State \hspace{-4mm}{Get generalized ridge leverage scores of $A$ w.r.t. $\tilde{C}'$}
\State \hspace{-4mm}{Use estimates to sample columns of $A$ to form $C$}
\State \hspace{-4mm}\Return $C$
\EndProcedure
\end{algorithmic}
\end{algorithm}

\noindent{\bf Improved Running Time.}
With a similar argument as for least squares regression, we obtain the following theorem. 
\begin{theorem}\label{nnza_intro} 
There is an iterative column sampling algorithm that, in time $O\left(T(A)\log n + n  \cdot \poly(k/\epsilon)\right)$, returns ${Z} \in \mathbb{R}^{n\times k}$ satisfying:
$\|A-ZZ^T A\|_F^2 \le (1+\epsilon) \|{A - {A}_k}\|_F^2.$
\end{theorem}

{\bf $\ell_p$-Regression.}
Another important problem is the {\it $\ell_p$-regression} problem. 
Given $A \in \mathbb{R}^{n \times d}$ and $b\in \mathbb{R}^{n \times 1}$, we want to output an $x \in \mathbb{R}^d$ 
satisfying~\eqref{eqn:lp}. 
We first consider the problem: 
for $C=[A,b]\in \mathbb{R}^{n\times (d+1)}$, find a matrix $S$ such that for every $x\in\mathbb{R}^{(d+1)\times1}$,
\begin{align}\label{preserve l_p norm}
(1-\epsilon)\|Cx\|_p\leq\|SCx\|_p \leq (1+\epsilon)\|Cx\|_p.
\end{align}
The following {\sc{ApproxLewisForm}} Algorithm is given in~\cite{cohen2015p} to solve \eqref{preserve l_p norm}, and for
them it suffices to set the parameter $\theta$ in the algorithm description to a small enough constant. This is because
in Step 7 of {\sc{ApproxLewisForm}}, they run the algorithm of Theorem 4.4 in their paper, which runs in at most $n$ time
provided $\theta$ is a small enough constant and $n > d^{C'}$ for a large enough constant $C' > 0$. We refer the reader
to \cite{cohen2015p} for the details, but remark that by setting $\theta$ to be a constant, Step 5 of {\sc{ApproxLewisForm}}
can be implemented in $T(A)$ time. Also, due to space constraints, we do not define the quadratic form $Q$ in what follows;
the algorithm for computing it is also in Theorem 4.4 of \cite{cohen2015p}. The only property we need is that it is computable
in $m \log m \log \log m \cdot d^C$ time, for an absolute constant $C > 0$, 
if it is applied to a matrix with at most $m$ rows. Theorem 4.4 of \cite{cohen2015p} can be invoked with constant $\epsilon$,
giving the so-called Lewis weights up to a constant factor, after which one can sample $O(d (\log d)/\epsilon^2)$ rows according
to these weights. Note that our running time is $O(T(A) \log n + \poly(d/\epsilon))$ even for constant
$\theta$, since in each recursive call we may need to spend $T(A)$ time, unlike \cite{cohen2015p}, who obtain a geometric
series of $\nnz(A) + \nnz(A)/2 + \nnz(A)/4 + \cdots + 1 \leq 2\nnz(A)$ time in expectation. Here, we do not know if 
$T(A)$ decreases when looking at submatrices of $A$. 

\begin{algorithm}[H]
\caption{ApproxLewisForm}
\begin{algorithmic}[1]
\Procedure{ApproxLewisForm}{$C \in \mathbb{R}^{n \times (d+1)}$}
\State \hspace{-4mm}{If $n \leq d+1$, apply Theorem 4.4 in \cite{cohen2015p} to return $Q$. }
\State \hspace{-4mm}{Uniformly sample $n/2$ rows of $C$ to form $\widehat{C}$}
\State \hspace{-4mm}{Let $\widehat{Q} = \textsc{ApproxLewisForm}(\widehat{C}, p, \theta)$ }
\State \hspace{-4mm}{Let $u_i$ be an $n^{\theta / p}$ multiplicative approximation of $c_i^T \widehat{Q} c_i$}
\State \hspace{-4mm}{Nonuniformly sample rows of $C$, taking an expected $p_i = \min(1, f(p) n^{\theta / 2} d^{p/2} \log d u_i^{p/2})$ copies of row $i$ (each scaled down by $p_i^{-1/p}$), producing $C'$}
\State \hspace{-4mm}{Apply Theorem 4.4 in \cite{cohen2015p} to $C'$, and return the quadratic form $Q$. }
\State \hspace{-4mm}{\Return $Q$}
\EndProcedure
\end{algorithmic}
\end{algorithm}
Combined with sampling by Lewis
weights we obtain an approximation for $C$ with only $\poly(d/\epsilon)$ rows.
Applying our earlier arguments to this setting yields the following: 
\begin{theorem}\label{thm:lp}
Given $\epsilon \in (0, 1)$, a constant $p\geq 1$, $A \in \mathbb{R}^{n \times d}$ and $b\in \mathbb{R}^{n \times 1}$, there is an algorithm that, in time $O\left(T(A)\log n + \poly(d/\epsilon)\right)$, returns $\hat{x}\in \mathbb{R}^{d \times 1}$ such that 
\begin{align*}
\|A\hat{x}-b\|_p \leq(1+\epsilon)\min_{x}\|Ax-b\|_p.
\end{align*}
\end{theorem}

\section{Applications}
{\bf Autoregression and General Dynamical Systems.}
In the original AR(d) model, we have:
\begin{align}\label{eqn:basicVersion}
\begin{bmatrix}
b_{d+1} \\
b_{d+2}\\
\vdots\\
b_{n+d}
\end{bmatrix}
=
\begin{bmatrix}
    b_{d} & 
    \dots  & b_{1} \\
    b_{d+1} & 
    \dots  & b_{2} \\
    \vdots & \ddots & \vdots \\
    b_{n+d-1} & 
    \dots  & b_{n}
\end{bmatrix}
\begin{bmatrix}
x_1\\
x_2\\
\vdots\\
x_d
\end{bmatrix}
+
\begin{bmatrix}
\eps_{d+1} \\
\eps_{d+2}\\
\vdots\\
\eps_{n+d}
\end{bmatrix}
\end{align}

Here we can create an $n \times d$ matrix $A$ 
where the $i$-th row is $(b_{i+d-1}, b_{i+d-2}, \ldots, b_{i})$. One 
obtains the $\ell_2$-regression problem
$\min_x \|Ax-b\|_2$ with $b^T = (b_{d+1}, \ldots, b_{n+d})$. 
In order to apply Theorem \ref{thm:main}, we need to bound $T(A)$. The following
lemma follows from the fact that $A$ is a submatrix of a Toeplitz matrix. 

\begin{lemma}\label{lem:TA}
$T(A) = O(n \log n)$. 
\end{lemma}

Combining Lemma \ref{lem:TA} with Theorem \ref{thm:main}, we can conclude: 
\begin{theorem}\label{thm:auto}
Given an instance $\min_x \|Ax-b\|_2$ of autoregression, with probability at least $1-\delta$ one can find a vector $\hat{x}$ so that $\|A\hat{x}-b\|_2 \leq (1+\epsilon) \min_x \|Ax-b\|_2$ in total time
$$
O\left((n \log^2 n + (d^2 \log^2 n)/\epsilon^2 + d^3 (\log n)/\epsilon^2)\log(1/\delta)\right).
$$
\end{theorem}

{\bf General Dynamical Systems.}
When dealing with more general dynamical systems, 
the $A$ in Theorem~\ref{thm:auto} would become
$
A = TUD
$, 
where $T$ is a Toeplitz matrix, $U$ is a matrix that represents computing successive differences, and 
$
D = \diag\big\{1, \frac{1}{h},  \ldots, \frac{1}{h^{d-1}}\big\}. 
$
Note that $T$ is $n \times d$, as for linear dynamical systems, $U$ is $d \times d$ and the operation $xU$ corresponds to replacing $x$ with $(x_2-x_1, x_3-x_2, x_4-x_3, \ldots, x_d-x_{d-1}, 0)$, and $D$ is a $d \times d$ diagonal matrix, and so $U$ and $D$ can each be applied to a vector in $O(d)$ time. Consequently by Lemma \ref{lem:TA}, we still have
$T(A) \leq T(T) + T(U) + T(D) = O(n \log n)$, and
we obtain the same time bounds in Theorem \ref{thm:auto}. 

{\bf Kernel Autoregression.} 
Let $\phi:\mathbb{R}^p \rightarrow \mathbb{R}^{p'}$ be a kernel transformation, as
defined in the introduction. The kernel autoregression problem is:
$\phi(b_t) = \sum_{i=1}^d \phi(b_{t-i}) x_i + \epsilon_t,$
where now note that $\epsilon_t \in \mathbb{R}^{p'}$. 
Note that there are still only $d$ unknowns $x_1, \ldots, x_d$. One way
of solving this would be to compute 
$\phi(b_t)$ for each $t$, represented as a column vector in $\mathbb{R}^{p'}$, 
and then create the linear system by stacking such vectors on top of each other:
\begin{align}\label{eqn:kernel}
\begin{bmatrix}
\phi(b_{d+1}) \\
\phi(b_{d+2})\\
\vdots
\\\phi(b_{n+d})
\end{bmatrix}
=
\begin{bmatrix}
    \phi(b_{d+1}) & \dots  & \phi(b_{1}) \\
    \phi(b_{d+2})  & \dots  & \phi(b_{2}) \\
     \vdots & \ddots & \vdots 
     \\\phi(b_{n+d-1})  & \dots  & \phi(b_{n})
\end{bmatrix}
\begin{bmatrix}
x_1\\
x_2\\
\vdots\\
x_d
\end{bmatrix}
+
\begin{bmatrix}
\eps_{d+1} \\
\eps_{d+2}\\
\vdots
\\\eps_{n+d}
\end{bmatrix}
\end{align}

One can then compute a $(1+\epsilon)$-approximate least squares solution
to (\ref{eqn:kernel}). Now the 
design matrix $\phi(A)$ in the regression problem is the vertical concatenation
of $p'$ matrices $A$, and an analogous argument
shows that $T(\phi(A)) = O(np' \log (np'))$, which gives us the analogous
version of Theorem \ref{thm:auto}, showing least squares regression is solvable
in $O(np' \log^{2} (np')) + \poly((d \log n)/\epsilon)$ with constant probability. 
While correct, 
this is prohibitive since $p'$ may be large. 

{\bf Speeding up General Kernels.}
Let $\phi(A)$ denote the design matrix in (\ref{eqn:kernel}), 
where the $i$-th block is $\phi(A)^i=[\phi(b_{i+d-1}); \phi(b_{i+d-2}); \ldots; \phi(b_{i})]$. Here  
$b$ is $[\phi(b_{d+1}); \ldots; \phi(b_{n+d})]$, which we know. 
We first compute $\phi(A)^T \phi(A)$. To do so quickly, we again exploit
the Toeplitz structure of $A$. More specifically, we have that  
$\phi(A)^T \phi(A) = \sum_i (\phi(A)^i)^T \phi(A)^i.$
In order to compute $(\phi(A)^i)^T \phi(A)^i$, we must compute $d^2$
inner products, namely, 
$\langle \phi(b_{d-j+i}), \phi(b_{d-j'+i}) \rangle$ for all 
$j, j' \in \{1, 2, \ldots, d\}$. Using the kernel trick,
$\langle \phi(b_{d-j+i}), \phi(b_{d-j'+i}) \rangle = f(\langle b_{d-j+i}, b_{d-j'+i} \rangle)$
for some function $f$ that we assume can be evaluated in constant time, given
$\langle b_{d-j+i}, b_{d-j'+i} \rangle$. Note that the latter inner product can be computed
in $O(p)$ time and thus we can compute $(\phi(A)^i)^T \phi(A)^i$ for a given $i$, 
in $O(d^2 p)$ time. Thus, na\"ively, we can compute $\phi(A)^T\phi(A)$ in $O(nd^2 p)$ time.

We can reuse most of our computation across different blocks
$i$. As we range over all $i$, the inner products we compute
are those of the form $\langle \phi(b_{d-j+i}), \phi(b_{d-j'+i}) \rangle$ for 
$i \in \{1, \ldots, n\}$ and $j, j' \in \{1, 2, \ldots, d\}$. Although a na\"ive count
gives $nd^2$ different inner products, 
this overcounts since for each point $\phi(b_{d-j+i})$ 
we only need its inner product with $O(d)$ points
other than with itself, and so $O(nd)$ inner products in total.
This is total time $O(ndp)$. 

Given these inner products, we quickly evaluate $\phi(A)^T\phi(A)$. 
The crucial point is that not only is each entry in $\phi(A)^T \phi(A)$ 
a sum of $n$ inner products we already computed, but one can quickly determine
entries from other entries. Indeed, given an entry on one of the $2d-1$ diagonal bands, 
one can compute the next entry on the band in $O(1)$ time by subtracting off a single
inner product and adding one additional inner product, since two consecutive
entries along such a band share $n-1$ out of $n$ inner product summands. 

Thus, each diagonal can be computed
in $O(n + d)$ time, and so in total $\phi(A)^T \phi(A)$ can be computed in $O(nd + d^2)$
time, given the inner products. We can compute $\phi(A)^T \phi(A)$
in $O(ndp)$ time assuming $d \leq n$. 
We then define $R = (\phi(A)^T \phi(A))^{-1}$, which can be computed
in an additional $O(d^{\omega})$ time, where $\omega \approx 2.376$ is the exponent
of fast matrix multiplication. Thus, $R$ is computable in $O(ndp + d^{\omega})$ time. Note
this is optimal for dense matrices $A$, since just reading each entry of $A$ takes 
$O(ndp)$. 
We can compute $\phi(A)^Tb \in \mathbb{R}^d$ using the 
kernel trick, which takes $O(ndp)$ time. By the normal equations,
$x = R b$, which can be computed in $d^{\omega}$ time. Overall, we obtain
$O(ndp + d^{\omega})$ time. 

{\bf The Polynomial Kernel.}
We focus on the polynomial kernel of degree $q = 2$. 
Using the subspace embedding 
analysis for TensorSketch in \cite{anw14}, combined with a 
leverage score approximation algorithm in \cite{cw13}, we can find 
a matrix $R$ in $O(\nnz(A) + d^{\omega})$ 
time,
where $\omega \approx 2.376$ is the exponent of matrix multiplication, with 
the following guarantee: if we sample 
$O(d \log d + d/\epsilon)$ rows of $\phi(A)$ proportional to the squared row norms
of $\phi(A)   R$, forming a sampling and rescaling matrix $S$, then
$\|S\phi(A)x - Sb\|_2 = (1 \pm \epsilon) \|\phi(A)x - b\|_2$ simultaneously for all
vectors $x$. Here the $i$-th row of $S$ contains a $1/\sqrt{p_j}$ in the $j$-th entry if the
$j$-th row of $\phi(A)$ is sampled in the $i$-th trial, and the $j$-th entry is $0$ otherwise.
Here $p_j = \frac{\|e_j   \phi(A)   R\|_2^2}{\|\phi(A)   R\|_F^2}$, 
where $e_j$ is the
$j$-th standard unit vector. We show how to sample indices 
$i \in [n   p']$ proportional to the squared row norms of $\phi(A)   R$. 

Instead of sampling indices $i \in [np']$ proportional
to the exact squared row norms of $\phi(A)   R$, it is well-known (see, e.g., \cite{w14})
that it suffices to sample them proportional
to approximations $\tilde{\tau}_i$ to the actual squared row norms $\tau_i$, where
$\frac{\tau_i}{2} \leq \tilde{\tau}_i \leq 2 \tau_i$ for every $i$. As in \cite{DMMW12},
to do the latter, we can instead sample indices according to the squared row norms
of $\phi(A)   R   G$, where $G \in \mathbb{R}^{d \times O(\log n)}$ is a matrix
of i.i.d. Gaussian random variables. To do this, we can first compute $R   G$ in
$O(d^2 \log n)$ time, and now we must sample row indices proportional to the squared
row norms of $\phi(A) RG$. Note that if we sample an entry $(i,j)$ of
$\phi(A) RG$ proportional to its squared value, then the row index $i$
is sampled according to its squared row norm. Since $RG$ only has $O(\log n)$
columns $v^i$, we can do the following: we first approximate
the squared norm of each $\phi(A)   v^i$. Call our approximation $\gamma_i$ with
$\frac{1}{2} \|\phi(A)   v^i\|_2^2 \leq \gamma_i \leq 2 \|\phi(A)   v^i\|_2^2$. Since
we need to sample $s = O(d \log d + d/\epsilon)$ total entries, we sample each entry
by first choosing a column $i \in [d]$ with probability $\frac{\gamma_i}{\sum_{j=1}^d \gamma_j}$,
and then outputting a sample from column $i$ proportional to its squared value. 
We show (1) how to obtain the $\gamma_i$ and (2) how to obtain
$s$ sampled entries, proportional to their squared value, 
from each column $\phi(A)   v^i$. 

For the polynomial kernel of degree $2$, 
the matrix 
$\phi(A)$ is in $\mathbb{R}^{np^2 \times d}$, since each
of the $n$ points is expanded to $p^2$ dimensions by $\phi$. 
Then $\phi(A)$ is the vertical concatenation of $B^1, \ldots, B^{n}$,
where each $B^i \in \mathbb{R}^{p^2 \times d}$
is a subset of columns of the matrix 
$C^i \circ C^i \in \mathbb{R}^{p^2 \times d^2}$, where 
$C^i \in \mathbb{R}^{p \times d}$ and $C^i \circ C^i$ consists of the Kronecker
product of $C^i$ with itself, i.e., the $((a,b),(c,d))$-th entry
of $C^i \circ C^i$ is $C^i_{a,b}   C^i_{c,d}$. Notice that $B^i$ consists
of the subset of $d$ columns of $C^i$ corresponding to $b = d$. 
Fix a column vector $v \in \{v^1, \ldots, v^d\}$ 
defined above. Let $S$ be the TensorSketch of \cite{pp13,anw14} with $O(1)$ rows. Then $S   \phi(z)$
can be computed in $O(\nnz(z))$ time for any $z$. For block $B^i$, 
$\|S   B^i v\|_2^2 = (1 \pm 1/10)\|B^i v\|_2^2$ with probability at least $2/3$. We can
repeat this scheme $O(\log n)$ times independently, creating a new matrix $S$ with $O(\log n)$
rows, for which $S   \phi(b_i)$ can be computed in $O(\nnz(b_i) \log n)$ time and thus
$S   \phi(A)$ can be computed in $O(\nnz(A) \log n)$ time overall. Further, we have the property that
for each block $B^i$, 
$\|S   B^i v\|_{med} = (1 \pm (1/10))\|B^i v\|_2^2,$
with probability $1-1/n^2$, where the med   operation denotes
taking the median estimate on each of the $O(\log n)$ independent repetitions. 
By a union bound, with probability 
$1- O(1/n)$,  
\begin{align}\label{eqn:median}
\|S   B^i v\|_{\text{med}} = (1 \pm (1/10))\|B^i v\|_2^2, 
\end{align}
simultaneously for every $i = 1, \ldots, n$. Notice $\phi(A)$ is a 
{\it block-Toeplitz matrix}, truncated to its first $d$ columns, 
where each block corresponds to $\phi(b_i)$ for some $i$. Suppose we replace
the blocks $\phi(b_i)$ with $S   \phi(b_i)$, obtaining a new block Toeplitz matrix $A'$, truncated to its first
$d$ columns, where now each
block has size $O(\log n)$. 
The new block Toeplitz matrix can be viewed as $O(\log n)$ disjoint standard
(blocks of size $1$) Toeplitz matrices with $n$ rows, and truncated to their first $d$ columns. Thus, 
$T(A') = O(n \log^2 n)$. The $i$-th block of coordinates
of size $O(\log n)$ is equal to $S   B^i   v$, and by (\ref{eqn:median}) we can in $O(\log n)$ time
compute a number $\ell_i = (1 \pm (1/10))\|B^i v\|_2^2$. Since this holds for every $i$, we can compute the desired estimate $\gamma_j$ to $\|\phi(A)v_j\|_2^2$ if $v = v_j$. The time is $O(\nnz(A)\log n + n \log^2 n)$.

After computing the $\gamma_j$, suppose our earlier sampling scheme samples $v = v_j$. Then to output an entry
of $\phi(A)v$ proportional to its squared value, we first output a block $B^i$ proportional to $\|B^i v\|_2^2$. Note that
given the $\ell_i$, we can sample such a $B^i$ within a $(1 \pm 1/10)$ factor of the actual sampling probability. Next,
we must sample an entry of $B^i v$ proportional to its squared value. 
The entries of $B^i   v$ are in one-to-one correspondence
with the entries of $C^i D_v (C^i)^T$, where $D_v \in \mathbb{R}^{d \times d}$ is the diagonal matrix
with the entries of $v$ along the diagonal. Let $H \in \mathbb{R}^{O(\log n) \times p}$ be a matrix
of i.i.d. normal random variables. We first compute $H   C^i$. This
can be done in $O(p d \log n)$ time. We then compute $H   C^i   D_v$ in $O(d \log n)$ time, 
and then $(H   C^i   D_v)   (C^i)^T$ in $O(pd \log n)$ time. By the Johnson-Lindenstrauss
lemma (see, e.g., \cite{johnson1984extensions}), each squared column norm of $H   C^i   D_v   (C^i)^T$ is the same as that of
$C^i   D_v   (C^i)^T$ up to a factor of $(1 \pm 1/10)$, for an appropriate $O(\log n)$ number
of rows of $H$. So we first sample a column $j$ of $C^i   D_v   (C^i)^T$ proportional to this
approximate squared norm. Next we compute $C^i   D_v   (C^i)^T e_j$ in $O(pd + d^2)$
time, and then in $O(p)$ time we output an entry of the $j$-th column proportional to its squared value.
Thus we find our sample.

To bound the overall time, note that we only need to compute the $\gamma_j$ values once, for $j = 1, \ldots, O(\log n)$
and for each $j$ this takes $O(\nnz(A) \log n + n \log^2 n)$ time. So in total across all indices $j$ this takes
$O(\nnz(A) \log^2 n + n \log^3 n)$ time. Moreover, this procedure also gave us the values $\ell_i$ for each $v^j$. Suppose
we also sort the partial sums $\sum_{i = 1}^{i'} \ell_i$ for each $1 \leq i' \leq n$ and corresponding to each $v^j$. This
takes $O(n \log^2 n)$ time and fits within our time bound. Then for
each of our $O(d \log d + d/\epsilon)$ samples we need to take, we can first sample $j$ based on the $\gamma_j$ values
and sample $i$ based on the sorted partial sums of $\ell_i$ values in $O(\log n)$ time via a binary search. Having
found $i$, we perform the procedure in the previous paragraph which takes $O(pd \log n + d^2)$ time. Thus, the
time for sampling is $O((p d^2 \log n + d^3 )(1/\epsilon + \log d))$. 

The overall time is, up to a constant factor, 
$O\big(\nnz(A) \log^2 n + n \log^3 n + (pd^2 \log n + d^3)(\frac 1 \epsilon + \log d)\big).$

\section*{Acknowledgement} 
D. Woodruff thanks partial support from NSF No. CCF-1815840. 
Part of the work was done while the authors were visiting the Simons Institute for the Theory of Computing. 

\bibliographystyle{aaai}
\bibliography{sketch}

\end{document}